\documentclass{elsarticle}
\usepackage{amsmath,amssymb,amsfonts}
\usepackage{times}
\usepackage{indentfirst}
\usepackage{graphicx}
\usepackage{subfigure}
\usepackage{amssymb}
\usepackage{algorithm}
\usepackage{algorithmic}
\newtheorem{theorem}{Theorem}
\newtheorem{definition}{Definition}
\newtheorem{proposition}{Proposition}
\newtheorem{corollary}{Corollary}
\newtheorem{lemma}{Lemma}
\newtheorem{example}{Example}
\newdefinition{remark}{Remark}
\newproof{proof}{Proof}

\makeatletter
\newif\if@borderstar
\def\bordermatrix{\@ifnextchar*{%
\@borderstartrue\@bordermatrix@i}{\@borderstarfalse\@bordermatrix@i*}%
}
\def\@bordermatrix@i*{\@ifnextchar[{\@bordermatrix@ii}{\@bordermatrix@ii[()]}}
\def\@bordermatrix@ii[#1]#2{%
\begingroup
\m@th\@tempdima8.75\p@\setbox\z@\vbox{%
\def\cr{\crcr\noalign{\kern 2\p@\global\let\cr\endline }}%
\ialign {$##$\hfil\kern 2\p@\kern\@tempdima &\thinspace %
\hfil $##$\hfil &&\quad\hfil $##$\hfil\crcr\omit\strut %
\hfil\crcr\noalign{\kern -\baselineskip}#2\crcr\omit %
\strut\cr}}%
\setbox\tw@\vbox{\unvcopy\z@\global\setbox\@ne\lastbox}%
\setbox\tw@\hbox{\unhbox\@ne\unskip\global\setbox\@ne\lastbox}%
\setbox\tw@\hbox{%
$\kern\wd\@ne\kern -\@tempdima\left\@firstoftwo#1%
\if@borderstar\kern2pt\else\kern -\wd\@ne\fi%
\global\setbox\@ne\vbox{\box\@ne\if@borderstar\else\kern 2\p@\fi}%
\vcenter{\if@borderstar\else\kern -\ht\@ne\fi%
\unvbox\z@\kern-\if@borderstar2\fi\baselineskip}%
\if@borderstar\kern-2\@tempdima\kern2\p@\else\,\fi\right\@secondoftwo#1$%
}\null \;\vbox{\kern\ht\@ne\box\tw@}%
\endgroup
}
\makeatother

\begin{document}

\begin{frontmatter}
\title{Connectedness of graphs and its application to connected matroids through covering-based rough sets}
\author{Aiping Huang}
\author{William Zhu\corref{cor1}}
\ead{williamfengzhu@gmail.com}
\address{Lab of Granular Computing,\\
Minnan Normal University, Zhangzhou 363000, China}
\cortext[cor1]{Corresponding author.}

\begin{abstract}
Graph theoretical ideas are highly utilized by computer science fields especially data mining.
In this field, a data structure can be designed in the form of tree.
Covering is a widely used form of data representation in data mining and covering-based rough sets provide a systematic approach to this type of representation.
In this paper, we study the connectedness of graphs through covering-based rough sets and apply it to connected matroids.
First, we present an approach to inducing a covering by a graph, and then study the connectedness of the graph from the viewpoint of the covering approximation operators.
Second, we construct a graph from a matroid, and find the matroid and the graph have the same connectedness, which makes us to use covering-based rough sets to study connected matroids.
In summary, this paper provides a new approach to studying graph theory and matroid theory.
\end{abstract}


\end{frontmatter}

\section{Introduction}
In different areas, various applications are addressed using graph models.
This model arrangements of various objects or technologies lead to new inventions and modifications in the existing environment for enhancement in those fields.
Connected graph, as an important concept of graph theory, is used in iatrology to study the spread of epidemics in a crowd where the vertices represent the persons in the crowd and the edges represent the spread of disease.
This model is important for tracking the spread of the disease, thus conducive to controlling it.
Just because graph theory can be used to modeling various applications, it is highly utilized by computer science applications.
Especially in data mining~\cite{Nettleton13Data,ShelokarQuirinCord¨®n13Amultiobjective}, image segmentation~\cite{PengZhangZhang13Asurvey,ZhouZhengWei13Texture}, clustering~\cite{Fukami14New,GossenKotzybaN¨¹rnberger14Graph}, networking~\cite{ChudnovskyRiesZwolsc11Claw}.

In matroid theory~\cite{Oxley93Matroid}, there are many terminology borrowed from graph theory, largely because it is an abstraction of various notions in the field.
The connectedness for matroids, which is extended by the corresponding notion for graph, is closely linked with the connectedness for graphs.
When the matroid is the cycle matroid induced by a graph, the matroid and the graph have the same connectedness.
It lays a sound foundation for us to apply graphs to study the connectedness for matroids.
In addition to that, matroids provide well-established platforms for greedy algorithms~\cite{Edmonds71Matroids} which may come with the algorithms for graphs.
The reasons given above are the motivations behind the study of the connected matroids from the perspective of connected graphs.

In this paper, we pay our attention to the connectedness for graphs through covering-based rough sets and apply it to connected matroids.
First, we introduce an approach to inducing a covering from a graph.
Based on the covering, covering-based rough set theory is used to study the issue of the connectedness for the graph.
As an application, we use the connection in graphs to study the connectedness for matroids.
In this part, we construct a graph from a matroid and find that they have the same connectedness, which makes us use covering-based rough sets to study the connection of matroids.
In a word, this work provides new viewpoints for studying graph theory and matroid theory.

\section{Preliminaries}
\label{S:Preliminaries}

To facilitate our discussion, some fundamental concepts related to covering-based rough sets, graphs and matroids are reviewed in this section.

\subsection{Covering-based rough sets}
As a generalization of a partition, the covering has more applicability and universality.
To begin with, the concept of covering is introduced.

\begin{definition}(Covering~\cite{ZhuWang03Reduction})
Let $U$ be a universe of discourse and $C$ a family of subsets of $U$.
If none of the subsets in $C$ are empty and $\bigcup C = U$, then $C$ is called a covering of $U$ and the pair $(U, C)$ is called a covering approximation space.
\end{definition}

As the two key concepts of covering-based rough sets, the lower and upper approximation operators are defined to describe objects.

\begin{definition}(Approximation operators~\cite{pomykala1987approximation})
Let $C$ be a covering of $U$ and $X \subseteq U$.
The covering upper and lower approximations of $X$, denoted by $\overline{C}(X)$ and $\underline{C}(X)$, respectively, are defined as:
\begin{center}
$~~~~~~~~~~~~~\overline{C}(X) = \bigcup \{K \in C: K \bigcap X \neq \emptyset\}$,\\
$\underline{C}(X) = \overline{C}(X^{c})^{c},~~~~~~~~~~~~~~~~$
\end{center}
where $X^{c}$ denotes the complement of $X$ in $U$.
\end{definition}

Immediately following the above definition, certain properties of the covering upper approximation operator are presented, while the corresponding properties of the covering lower one can be obtained by the duality property.

\begin{proposition}\cite{pomykala1987approximation}
\label{P:thepropertiesofcoveringupperapproximationoperator}
Let $C$ be a covering of $U$.
The operator $\overline{C}$ has the following properties:\\
(1) $\overline{C}(\emptyset) = \emptyset$~~~~~~~~~~~~~~~~~~~~~~~~~~~~~~~~~~~~~~~~~~~~~~~~~~~~~~~~~~~~~~~~~~~(Normality).\\
(2) $\overline{C}(U) = U$~~~~~~~~~~~~~~~~~~~~~~~~~~~~~~~~~~~~~~~~~~~~~~~~~~~~~~~~~~~~~~~~(Co-normality).\\
(3) For all $X \subseteq U$, $X \subseteq \overline{C}(X)$~~~~~~~~~~~~~~~~~~~~~~~~~~~~~~~~~~~~ (Extension).\\
(4) For all $X, Y \subseteq U$, $\overline{C}(X \bigcup Y) = \overline{C}(X) \bigcup \overline{C}(Y)$~~~(Additivity).\\
(5) If $X \subseteq Y \subseteq U$, then $\overline{C}(X) \subseteq \overline{C}(Y)$~~~~~~~~~~~~~~~~~~~~~~(Monotonicity).
\end{proposition}

\subsection{Graphs}
Graphs are discrete structures to model the correlation between data.
Theoretically, a graph is a pair $G = (V, E)$ comprising a set $V$ of vertices and a set $E$ of edges~\cite{West04Introduction}.
Generally, we write $V(G)$ for $V$ and $E(G)$ for $E$, particularly when several graphs are considered.
Each element of $E(G)$ has either one or two vertices associated with it, called its endpoints.
Through endpoints, the relationship between vertices and edges can be established by the form of matrices, namely incidence matrices.
Let $G = (V, E)$ be a graph with $V = \{v_{1}, v_{2}, \cdots, v_{n}\}$ and $E = \{e_{1}, e_{2}, \cdots, e_{m}\}$.
The incidence matrix $I(G)$ of the graph is the $n \times m$ matrix in which entry $m_{ij}$ is $1$ if $v_{i}$ is an endpoint of $e_{j}$ and otherwise is $0$.

The edges in a graph may be directed or undirected.
If any edge of the graph is undirected, we say the graph is an undirected graph.
In this case, we write $e = uv$ or $e = vu$ for an edge $e$ with endpoints $u$ and $v$.
In a graph $G$, two vertices are adjacent if there is an edge that has them as endpoints.
An isolated vertex is a vertex not adjacent to any other vertices.
If an edge links the same two endpoints, the edge is called a loop, and if there are edges having the same pair of endpoints, they are called multiple edges.
A simple graph is a graph without loops or multiple edges.
If the graph is simple and the vertices of it are pairwise adjacent, it is called the complete graph.
A subgraph of the graph $G$ is a graph whose vertices and edges are subsets of $G$.
The subgraph induced by a subset of vertices $K \subseteq V(G)$ is called a vertex-induced subgraph of $G$, and denoted by $G_{K}$.
This subgraph has vertex set $K$, and its edge set $E^{'} \subseteq E(G)$ consists of those edges from $E(G)$ that have both their endpoints in $K$.

\begin{example}
Let $G = (V, E)$ be a graph as given in (I) of Figure \ref{anundirectedgraph}.
Suppose $K = \{b, c, d\}$.
Then the vertex-induced subgraph $G_{K}$ is shown in (III) of Figure \ref{anundirectedgraph}.
\end{example}

A path of a graph $G$ is a list $v_{0} v_{1} \cdots v_{k}$ of distinct vertices such that, for all $1 \leq i \leq k$, $v_{i - 1} v_{i}$ is an edge of $G$, and a $(u, v)-$path is a path and has first vertex $u$ and last vertex $v$.

A graph $G$ is connected if for every pair of distinct vertices $u$ and $v$, there is a path connecting both.
If $G$ has a $(u, v)-$path, then $u$ is connected to $v$.
The connection relation on $V(G)$ consists of the order pairs $(u, v)$ such that $u$ is connected to $v$.
It was noted in \cite{West04Introduction} that the connection relation is an equivalence relation on $V(G)$.
Suppose the equivalence classes of the relation are $V_{1}, V_{2}, \cdots, V_{s}$.
Then the vertex-induced subgraphs $G_{V_{1}}, G_{V_{2}}, \cdots, G_{V_{s}}$ are called the connected components of the graph.
The number of the connected components of graph $G$ is denoted by $\omega(G)$.

\subsection{Matroids}
Matroid theory borrows extensively from the terminology of graph theory,
largely because it is an abstraction of various notions of central importance in the field, such as independent sets and circuits.
The following definition of matroids is presented in terms of circuits.

\begin{proposition}(Circuit axiom~\cite{Oxley93Matroid})
\label{P:circuitaxiom}
Let $\mathcal{C}$ be a family of subsets of $U$.
There exists a matroid $M$ such that $\mathcal{C}=\mathcal{C}(M)$ if and only if
$\mathcal{C}$ satisfies the following conditions:\\
(C1) $\emptyset \notin \mathcal{C}$.\\
(C2) For all $C_{1}, C_{2} \in \mathcal{C}$, if $C_{1} \subseteq C_{2}$, then $C_{1} = C_{2}$.\\
(C3) For all $C_{1}, C_{2} \in \mathcal{C}$, if $C_{1} \neq C_{2}$ and $x \in C_{1} \bigcap C_{2}$, then there exists
$C_{3} \in \mathcal{C}$ such that $C_{3} \subseteq C_{1} \bigcup C_{2}-\{x\}$.
\end{proposition}

If the family $\mathcal{C}$ of subsets of $U$ satisfies the circuit axiom, then the members of $\mathcal{C}$ are called the circuits of $M$ and $U$ is called the ground set of $M$.
We often write $\mathcal{C}(M)$ for $\mathcal{C}$ and $U(M)$ for $U$, particularly when several matroids are being considered.
For a matroid $M$, if $C \in \mathcal{C}(M)$ and $C = \{x\}$, we say $x$ is a loop of the matroid.
If $\mathcal{C}(M)$ dose not contain any single-point set, the matroid is loopless.
By the family of circuits, the connected matroids are defined.
For any two elements $e, f$ of $U(M)$, define the relation $\gamma$ of $U(M)$ by $e \gamma f$ if and only if $e = f$ or $M$ has a circuit containing $e$ and $f$.
In~\cite{Oxley93Matroid}, it was indicated that the relation $\gamma$ is an equivalence relation.
For any $e \in U(M)$, the $\gamma-$equivalence class $\gamma(e) = \{e\} \bigcup \{f \in U(M): M$ has circuit containing $e$ and $f\}$ is called a connected component of $M$.
If $M$ has only one connected component $U(M)$, we call $M$ is connected; Otherwise $M$ is disconnected.
In fact, we can also describe the connected graph by the following proposition.

\begin{proposition}\cite{Oxley93Matroid}
\label{P:anequivalencecharacterizationforconnectedgraph}
The matroid $M$ is connected if and only if, for every pair of distinct elements of $U(M)$, there is a circuit containing both.
\end{proposition}

\section{The study of the connectedness of graph through covering-based rough sets}
\label{S:Thestudyofconnectioningraphsthroughcovering-basedroughsets}
Connected graphs are important discrete structures.
Problems in many fields can be addressed using the graph models.
In this section, we apply the covering-based rough sets to study the issue of the connection in a graph.
Considering the matrix is related significantly to the theory, the incidence matrices of a graph are also borrowed to study the issue.

In this subsection, we present certain approaches to judging whether a graph is connected or not through covering approximation operators.
For this purpose, we need to establish a relationship between a graph and a covering.
In~\cite{WangZhu13Equivalent}, it proposed the following approach to converting a graph to a covering.

\begin{definition}
\label{D:thefamilyinducedbygraphs}
Let $G = (V, E)$ be an undirected simple graph.
One can define a family $F(G)$ of subsets of $V$ as follows: For all $u, v \in V$,
\begin{center}
   $\{u, v\} \in F(G) \Leftrightarrow uv \in E$.
\end{center}
\end{definition}

Definition \ref{D:thefamilyinducedbygraphs} indicates that a graph can be represented by a family of subsets of its vertices.
However, the family may not be a covering of the vertex set.

\begin{example}
\label{E:Example1}
Let $G = (V, E)$ be the graph as given in (I) of Figure \ref{anundirectedgraph} where $V = \{a, b ,c , d, e\}$ and $E = \{e_{1}, e_{2}, e_{3}, e_{4}, e_{5}\}$.
\begin{figure}[h]
   \begin{center}
   \includegraphics[width = 4.5 in]{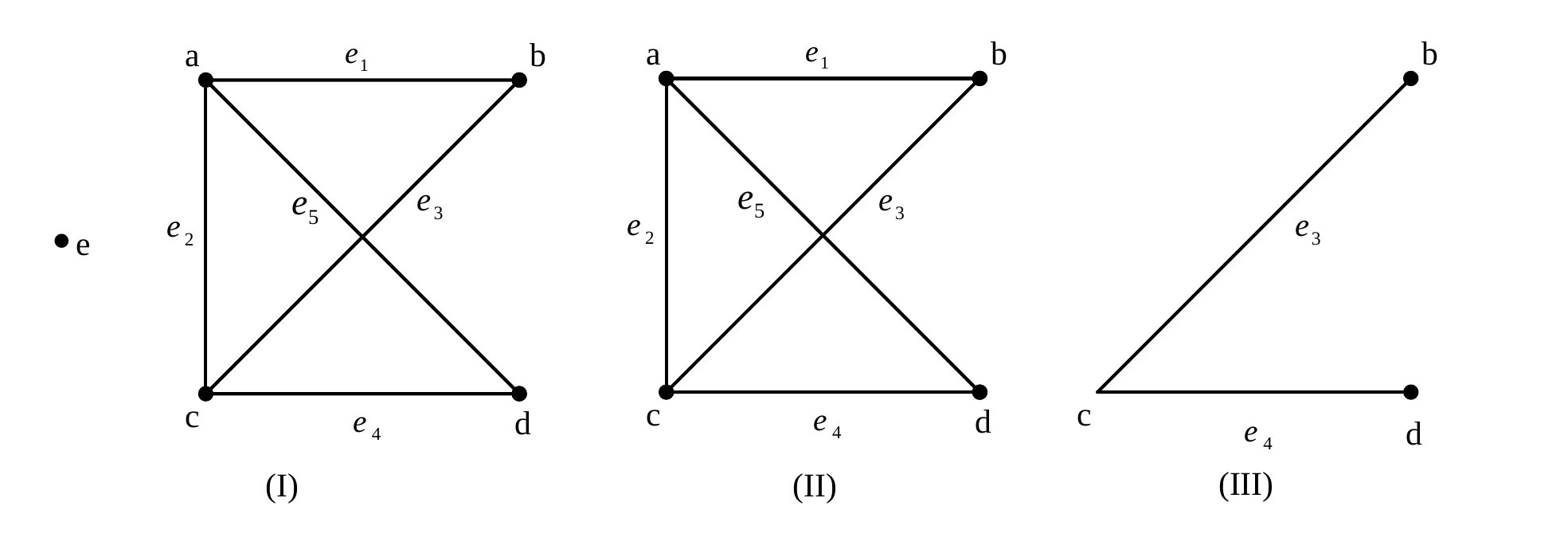}
   \caption{An undirected simple graph.}
   \label{anundirectedgraph}
    \end{center}
\end{figure}
By Definition \ref{D:thefamilyinducedbygraphs}, we know $F(G) = \{\{a, b\}$, $\{a, c\}$, $\{b, c\}$, $\{c, d\}$, $\{d, a\}\}$ and it is not a covering of $V$ because there does not exist any edge to connect with the vertex $e$.
\end{example}

In fact, the type of graph which can induce a covering was also embodied in \cite{WangZhu13Equivalent}.

\begin{proposition}
Let $G = (V, E)$ be an undirected simple graph.
The family $F(G)$ is a covering of $V$ if and only if $G$ has no isolated vertices.
\end{proposition}

Therefore, the graphs studied in this section are undirected, simple and without isolated vertices unless otherwise specified.
For the type of graph $G$, we denote the covering induced from it by $C(G)$.
First, the connection of any pair of distant vertices of the graph is studied by the covering.

\begin{proposition}
\label{P:thefirstequivalentcharacterizationforverticesconnection}
Let $G = (V, E)$ be a graph and $u, v$ be two distinct vertices of $V$.
The vertex $u$ is connected to $v$ if and only if $\{u , v\} \in C(G)$ or there exists $\{K_{1}, K_{2}, \cdots, K_{n}\}$ $\subseteq C(G)$ satisfying $u \in K_{1}$, $v \in K_{n}$ and $K_{i} \bigcap K_{i + 1} \neq \emptyset$ for any $i = 1, 2, \cdots, n - 1$.
\end{proposition}

\begin{proof}
$(``\Rightarrow")$: If $u$ and $v$ are adjacent, then $\{u, v\} \in C(G)$.
If $u$ is connected to $v$ but they are not adjacent, then there exists a $(u, v)-$path $u_{1}u_{2}u_{3}$ $\cdots u_{n - 1}u_{n}$ $ u_{n + 1}$, where $u = u_{1}$ and $v = u_{n + 1}$.
Let $K_{i} = \{u_{i}, u_{i + 1}\}$ for any $i = 1, 2, \cdots n$.
It is clear that $\{K_{1}, K_{2}, \cdots, K_{n}\} \subseteq C(G)$.
And $u \in K_{1}$, $v \in K_{n}$ and $K_{i} \bigcap K_{i + 1} \neq \emptyset$ for any $i = 1, 2, \cdots, n - 1$ because $u_{i + 1} \in K_{i} \bigcap K_{i + 1}$.

$(``\Leftarrow")$: For $u, v \in V$, if $\{u, v\} \in C(G)$, then $u$ and $v$ are adjacent.
If there exists $\{K_{1}, K_{2}, \cdots, K_{n}\} \subseteq C(G)$ satisfying $u \in K_{1}$, $v \in K_{n}$ and $K_{i} \bigcap K_{i + 1} \neq \emptyset$ for any $i = 1, 2, \cdots, n - 1$, then let $u_{i} \in K_{i} \bigcap K_{i + 1}$ where $i = 1, 2, \cdots, n - 1$.
Then there exists a list $u_{0} u_{1} u_{2} \cdots u_{n - 1} u_{n}$ connecting $u$ and $v$, where $u_{0} = u$ and $u_{n} = v$.
If there exist $i, j \in \{0, 1, \cdots, n\}$ (we may as well suppose $i < j$) such that $u_{i} = u_{j}$, then delete the the vertices $u_{i + 1}, u_{i + 2}, \cdots u_{j}$ of the list.
Finally, we can obtain a $(u,v)-$path.
Therefore $u$ is connected to $v$.
\end{proof}

Based on the above proposition, an equivalent characterization for connected graph is established in terms of the covering induced by the graph.

\begin{theorem}
\label{T:thefirstnecessaryandsufficientconditionforconnectedgraph}
Let $G = (V, E)$ be a graph.
The graph is connected if and only if, for any pair of distinct vertices $u$ and $v$ of $V$, $\{u, v\} \in C(G)$ or there exists $\{K_{1}, K_{2}, \cdots, K_{n}\}$ $\subseteq C(G)$ satisfying $u \in K_{1}$, $v \in K_{n}$ and $K_{i} \bigcap K_{i + 1} \neq \emptyset$ for any $i = 1, 2, \cdots, n - 1$.
\end{theorem}

\begin{proof}
It is straightforward from Proposition \ref{P:thefirstequivalentcharacterizationforverticesconnection} and the definition of connected graphs.
\end{proof}

\begin{example}
\label{E:anexampleforconnectedgraph}
Let $G = (V, E)$ be the graph as given in (II) of Figure \ref{anundirectedgraph} where $V = \{a, b, c, d\}$ and $E = \{e_{1}, e_{2}, e_{3}, e_{4}, e_{5}\}$.
Then the covering induced by $G$ is $C(G) = \{K_{1}$, $K_{2}$, $K_{3}$, $K_{4}$, $K_{5}\}$, where $K_{1} = \{a, b\}$, $K_{2} = \{a, c\}$, $K_{3} = \{b, c\}$, $K_{4} = \{c, d\}$ and $K_{5} = \{d, a\}$.
For the pair of distinct vertices $b$ and $d$ of $V$, there exists $\{K_{3}, K_{4}\} \subseteq C(G)$ satisfying $b \in K_{3}$ and $d \in K_{4}$ and $K_{3} \bigcap K_{4} \neq \emptyset$, thus $b$ is connected to $d$.
In the same way, we find that, for any two distinct vertices of $V$, they are connected, i.e., $G$ is a connected graph.
\end{example}

In fact, utilizing Theorem \ref{T:thefirstnecessaryandsufficientconditionforconnectedgraph}, the connected graphs can also be characterized equivalently from the viewpoint of the covering upper approximation operator.

\begin{theorem}
\label{T:thesecondnecessaryandsufficientconditionforconnectedgraphfromupperoperator}
Let $G = (V, E)$ be a graph.
The graph is connected if and only if, for any $\emptyset \neq X \subset V$, $\overline{C(G)}(X) \neq X$.
\end{theorem}

\begin{proof}
$(``\Rightarrow"):$ Since $\overline{C(G)}(\emptyset) = \emptyset$, we need to prove only the result: $\forall~\emptyset \neq X \subseteq V$, if $\overline{C(G)}(X) = X$, then $X = V$.
Pitch $u \in X$.
For all $v \in V - \{u\}$, if $\{u, v\} \in C(G)$, then $v \in \{u, v\} \subseteq \overline{C(G)}(X) = X$ which implies $V - \{u\} \subseteq X$.
If $u$ and $v$ are not adjacent, then there exists $\{K_{1}, K_{2}, \cdots, K_{n}\} \subseteq C(G)$ satisfying $u \in K_{1}$, $v \in K_{n}$ and $K_{i} \bigcap K_{i + 1} \neq \emptyset$ for any $i = 1, 2, \cdots, n - 1$.
Since $u \in K_{1} \bigcap X$, $K_{1} \subseteq \overline{C(G)}(X) = X$.
Because $K_{2} \bigcap K_{1} \neq \emptyset$, $K_{2} \subseteq \overline{C(G)}(K_{1})$.
Combining with the monotonicity of $\overline{C(G)}$, we have $K_{2} \subseteq \overline{C(G)}(K_{1}) \subseteq \overline{C(G)}(\overline{C(G)}(X)) = \overline{C(G)}(X) = X$.
In the same way, we can obtain $v \in K_{n} \subseteq X$, then $V - \{u\} \subseteq X$.
Since $u \in X$, $V \subseteq X$.
Combining with $X \subseteq V$, then $X = V$.

$(``\Leftarrow")$: For all $u \in V$, let $P_{u} = \{v \in V:$ $v$ is connected to $u\}$.
Then $P_{u} \neq \emptyset$ because $u \in P_{u}$.
Next, we want to prove $\overline{C(G)}(P_{u}) = P_{u}$.
For all $v \in \overline{C(G)}(P_{u})$, there exists $K \in C(G)$ such that $v \in K$ and $K \bigcap P_{u} \neq \emptyset$.
We may as well suppose $w \in K \bigcap P_{u}$, then $v$ is connected to $w$ and $w$ is connected to $u$, thus $v$ is connected to $u$, i.e., $v \in P_{u}$.
Thus $\overline{C(G)}(P_{u}) \subseteq P_{u}$.
Utilizing the extension of $\overline{C(G)}$, we have $P_{u} \subseteq \overline{C(G)}(P_{u})$.
Then $\overline{C(G)}(P_{u}) = P_{u}$.
By assumption, we know $P_{u} = V$.
Therefore $G$ is connected.
\end{proof}

Given a covering approximation space $(U, C)$, for all $X \subseteq U$, if $\overline{C}(X) = X$, the set $X$ is called an outer definable set.
From the viewpoint, Theorem \ref{T:thesecondnecessaryandsufficientconditionforconnectedgraphfromupperoperator} indicates that a graph is connected if and only if the covering approximation space induced by the graph has no non-empty outer definable proper subset.

By the duality, the connected graph characterized by the covering lower approximation operator is presented as follows.

\begin{corollary}
\label{C:thesecondnecessaryandsufficientconditionforconnectedgraphfromloweroperator}
Let $G = (V, E)$ be a graph.
The graph is connected if and only if, for any $\emptyset \neq X \subset V$, $\underline{C}(X) \neq X$.
\end{corollary}

\begin{example}
Let $G = (V, E)$ be the graph as given in (II) of Figure \ref{anundirectedgraph}.
By simple computing, the outer definable subsets of covering approximation space $(V, C(G))$ are $\emptyset$ and $V$.
Hence $G$ is connected.
\end{example}

\section{An application to connected matroids}
\label{S:anapplicationtomatroids}
As is known, when the matroid is the cycle matroid induced by a graph, the matroid and the graph have the same connectedness.
However, for a given matroid, it may not be the cycle matroid of some graph.
Therefore, using cycle matroids to study the connectedness of matroids may not be effective.
In this section, we propose an approach to induce a graph from an arbitrary matroid.
It is interesting that the graph and the matroid have the same connectedness.
Therefore, the covering-based rough sets are used to study the connection of the matroid.
First, the method to convert a matroid to a graph is presented as follows.

\begin{definition}
\label{D:thegraphinducedbyamatroid}
Let $M$ be a matroid.
One can define an undirected simple graph $G(M) = (V, E)$ as follows:\\
(1) $V = U(M)$.\\
(2) For any $u, v \in V$ and $u \neq v$, $u v \in E \Leftrightarrow \exists C \in \mathcal{C}(M)~s.t.~\{u, v\} \subseteq C$.
\end{definition}

\begin{remark}
Once matroid $M$ has loops, the graph $G(M)$ has isolated vertices.
\end{remark}

\begin{example}
\label{E:theexampleofthegrphinducedbyamatroid}
Let $M$ be a matroid with $U(M) = \{1, 2, 3, 4, 5, 6, 7\}$ and $\mathcal{C}(M) = \{\{1, 2, 3\},$ $\{6\}$, $\{2, 4, 5\},$ $\{1, 3, 4, 5\}\}$.
The graph induced by the matroid $M$ is given in Figure \ref{thegraphinducedbyamtroid}.
\begin{figure}[h]
   \begin{center}
   \includegraphics[width = 2.3 in]{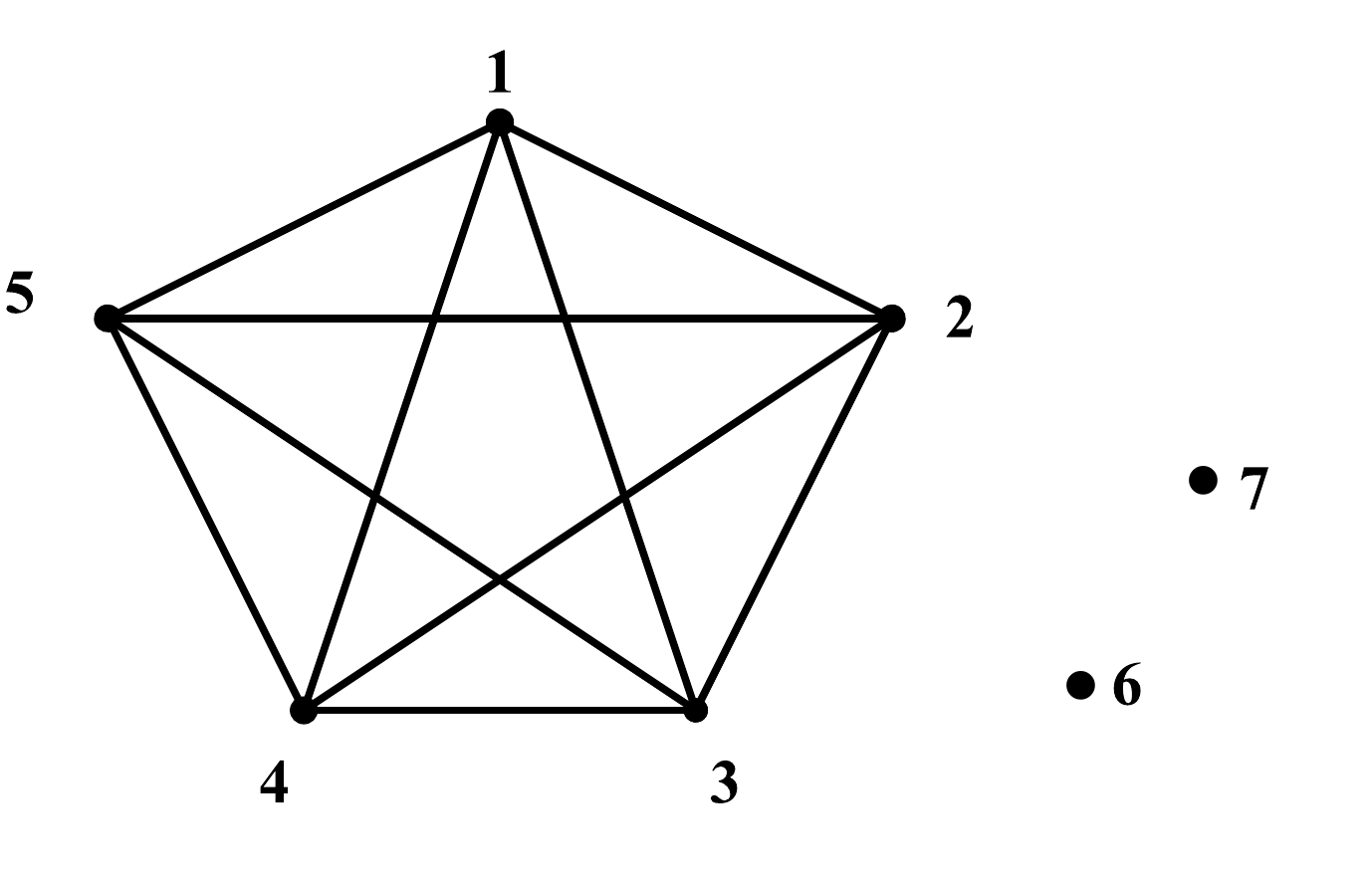}
   \caption{The graph $G(M)$ induced by $M$.}
   \label{thegraphinducedbyamtroid}
    \end{center}
\end{figure}
\end{example}

In fact, the connectedness of the graph induced by a matroid is closely related to that of the matroid.
First, the equivalent characterization for the connection of any pair of distinct vertices of the graph is presented through the circuits of the matroid.

\begin{lemma}\cite{Oxley93Matroid}
\label{L:thetransitivityofcircuits}
Let $M$ be a matroid and $C_{1}, C_{2} \in \mathcal{C}(M)$.
If $e_{1} \in C_{1} - C_{2}$, $e_{2} \in C_{2} - C_{1}$ and $C_{1} \bigcap C_{2} \neq \emptyset$, then there exists $C_{3} \in \mathcal{C}(M)$ such that $e_{1}, e_{2} \in C_{3} \subseteq C_{1} \bigcup C_{2}$.
\end{lemma}

\begin{proposition}
\label{P:anequivalencecharacterizationforanytwovertexicesofthegraphinducedbymatroid}
Let $M$ be a matroid and $u, v$ a pair of distinct vertices of $U(M)$.
The vertex $u$ is connected to $v$ in graph $G(M)$ if and only if there exists $C \in \mathcal{C}(M)$ such that $\{u, v\} \subseteq C$.
\end{proposition}

\begin{proof}
The sufficiency is straightforward.
Next, we prove the necessity.
Since $u$ is connected to $v$ in graph $G(M)$, there exists the shortest $(u, v)-$path, let us assume the length is $n$.
We conclude that $n = 1$.
Otherwise, we may well suppose the path is $u_{1} u_{2} \cdots u_{n + 1}$, where $u_{1} = u$, $u_{n + 1} = v$, and $n \geq 2$.
Since $u_{1} u_{2} \cdots u_{n + 1}$ is a path, there exist $C_{1}, C_{2}, \cdots, C_{n} \in \mathcal{C}(M)$ such that $\{u_{i}, u_{i + 1}\} \subseteq C_{i}$ for all $i = 1, 2, \cdots, n$.
Thus $u \in C_{1}$, $v \in C_{n}$ and $C_{i} \bigcap C_{i + 1} \neq \emptyset$ for all $i = 1, 2, \cdots, n - 1$.
Because the path is the shortest, the circuits $C_{1}, C_{2}, \cdots, C_{n}$ are different, and $u \notin C_{j}$ for all $j = 2, 3, \cdots, n$, and $v \notin C_{j}$ for all $j = 1, 2, \cdots, n - 1$.
As $u \in C_{1} - C_{2}$, according to $(C2)$ of the circuit axiom, we know there exists $v_{1} \in C_{2} - C_{1}$.
Combining with $C_{1} \bigcap C_{2} \neq \emptyset$ and Lemma \ref{L:thetransitivityofcircuits}, there exists $C_{1}^{'} \in \mathcal{C}(M)$ such that $\{u, v_{1}\} \subseteq C_{1}^{'} \subseteq C_{1} \bigcup C_{2}$.
Since $v_{1} \in C_{2}$, if $v_{1} \notin C_{3}$, then $v_{1} \in C_{2} - C_{3}$.
Utilizing $(C2)$ of the circuit axiom, there exists $v_{2} \in C_{3} - C_{2}$.
According to Lemma \ref{L:thetransitivityofcircuits} and $C_{2} \bigcap C_{3} \neq \emptyset$, there exists $C_{2}^{'} \in \mathcal{C}(M)$ such that $\{v_{1}, v_{2}\} \subseteq C_{2}^{'} \subseteq C_{2} \bigcup C_{3}$.
If $v_{1} \in C_{3}$, then we take $C_{2}^{'} = C_{3}$.
It is clear that $C_{1}^{'} \bigcap C_{2}^{'} \neq \emptyset$ because $v_{1} \in C_{1}^{'} \bigcap C_{2}^{'}$.
In the same way, we can obtain $C_{1}^{'}, C_{2}^{'}, \cdots, C_{n - 1}^{'} \in \mathcal{C}(M)$ such that $u \in C_{1}^{'}$, $v \in C_{n - 1}^{'}$ and $C_{i}^{'} \bigcap C_{i + 1}^{'} \neq \emptyset$ for all $i = 1, 2, \cdots, n - 2$.
Furthermore, $v \notin C_{j}^{'}$ for all $j = 1, 2, \cdots n - 2$.
If there exits $j \in \{1, 2, \cdots, n - 2\}$ such that $v \in C_{j}^{'}$, then $v \in C_{j}^{'} \subseteq C_{j} \bigcup C_{j + 1}$, i.e., $v \in C_{j}$ or $v \in C_{j + 1}$ which contradicts that $v \notin C_{j}$ for all $j \leq n - 1$.
Similarly, for all $j = 2, 3, \cdots, n - 1$, $u \notin C_{j}^{'}$.
Thus we have $C_{1}^{'} \neq C_{n - 1}^{'}$.
However, the circuits $C_{1}^{'}, C_{2}^{'}, \cdots, C_{n - 1}^{'}$ may not be all different.
Therefore, we reduce the circuits by the following step.
If there exist two distinct numbers $i, j$ of $\{1, 2, \cdots, n - 1\}$ such that $C_{i}^{'} = C_{j}^{'}$ (we may as well suppose $i < j$), then remove the circuits $C_{i + 1}^{'}, \cdots, C_{j}^{'}$.
By the step, we can obtain the family of circuits $\{C_{s_{1}}^{'}, C_{s_{2}}^{'}, \cdots, C_{s_{t}}^{'}\} (\subseteq \{C_{1}^{'}, C_{2}^{'}, \cdots, C_{n - 1}^{'}\})$ whose elements are different and satisfy the condition: $u \in C_{s_{1}}^{'}$, $v \in C_{s_{t}}^{'}$ and $C_{s_{i}} \bigcap C_{s_{i + 1}} \neq \emptyset$ for all $i = 1, 2, \cdots, t - 1$.
It is clear that $s_{1} = 1$, $s_{t} = n - 1$, $u \notin C_{s_{j}}^{'}$ for all $j = 2, 3, \cdots, t$ and $v \notin C_{s_{j}}^{'}$ for all $j = 1, 2, \cdots t - 1$.
For the circuits $C_{s_{1}}^{'}, C_{s_{2}}^{'}, \cdots, C_{s_{t}}^{'}$, repeat the above discussion.
Finally, we can obtain two circuits $C_{u}$ and $C_{v}$ such that $u \in C_{u} - C_{v}$, $v \in C_{v} - C_{u}$ and $C_{u} \bigcap C_{v} \neq \emptyset$.
Utilizing Lemma \ref{L:thetransitivityofcircuits}, there exists $C \in \mathcal{C}(M)$ such that $\{u, v\} \subseteq C \subseteq C_{u} \bigcup C_{v}$,
i.e., $uv \in E(G(M))$ which implies $n = 1$.
It contradicts the assumption that $n \geq 2$.
Hence the result has been proved.
\end{proof}

\begin{remark}
Any connected component of the graph $G(M)$ is an isolated vertex or a complete graph.
Once the graph is connected, it is a completed graph.
\end{remark}

By the above proposition, the relationship between the connectedness of a matroid and that of the graph induced by the matroid can be embodied.
We find that they have the same connectedness.

\begin{theorem}
\label{T:therelationshipbetweenamtroidandthegraphinducedbythematroid}
Let $M$ be a matroid.
The graph $G(M)$ is connected if and only if the matroid $M$ is connected.
\end{theorem}

\begin{proof}
According to Proposition \ref{P:anequivalencecharacterizationforconnectedgraph} and \ref{P:anequivalencecharacterizationforanytwovertexicesofthegraphinducedbymatroid},
$G(M)$ is connected $\Leftrightarrow$ for any two distant vertices $u$ and $v$ of $V(G(M))$, $u$ is connected to $v$ $\Leftrightarrow$ for any two distant vertices $u$ and $v$ of $U(M)$, there exists $C \in \mathcal{C}(M)$ such that $\{u, v\} \subseteq C$ $\Leftrightarrow$ $M$ is connected.
\end{proof}

Now that a matroid and the graph induced by the matroid have the same connectedness,
if the graph has isolated vertices, then the matroid is disconnected.
From Example \ref{E:theexampleofthegrphinducedbyamatroid}, we find that whether the graph has isolated vertices or not is not only determined by the loops of the matroid.
Indeed, it also has a relation with the other circuits of the matroid.

\begin{proposition}
\label{P:theconditionforthecircuitsofamatroidformacovering}
Let $M$ be a matroid.
The graph $G(M)$ has no isolated vertices if and only if $M$ is loopless and $\mathcal{C}(M)$ is a covering of $U(M)$.
\end{proposition}

\begin{proof}
$(``\Leftarrow")$: If $u$ is an isolated vertex of $G(M)$, then for all $v \in U(M) - \{u\}$, there does not exist $C \in \mathcal{C}(M)$ such that $\{u, v\} \subseteq C$.
Thus there exists $C \in \mathcal{C}(M)$ such that $C = \{u\}$ or there does not exist $C \in \mathcal{C}(M)$ such that $u \in C$, which contradicts the assumption that $M$ is loopless and $\mathcal{C}(M)$ is a covering of $U(M)$, respectively.

$(``\Rightarrow")$: It is clear that $G(M)$ has no isolated vertices implies that $M$ is loopless.
Next, we need to prove $\mathcal{C}(M)$ is a covering of $U(M)$.
According to the circuit axiom, we know $\emptyset \notin \mathcal{C}(M)$.
For all $u \in U(M)$, there exists an element of $U(M)$ which is different from $u$ such that $u$ is connected to $v$.
Utilizing Proposition \ref{P:anequivalencecharacterizationforanytwovertexicesofthegraphinducedbymatroid}, there exists $C_{u} \in \mathcal{C}(M)$ such that $u \in C_{u}$.
Thus $U(M) = \bigcup_{u \in U(M)}\{u\} \subseteq \bigcup_{u \in U(M)}C_{u} \subseteq \bigcup \mathcal{C}(M) \subseteq U(M)$, i.e., $U(M) = \bigcup \mathcal{C}(M)$.
Therefore $\mathcal{C}(M)$ is a covering of $U(M)$.
\end{proof}

Next, we pay our attention to the connectedness of the matroid whose induced graph has no isolated vertices.
In this part, we introduce the approaches proposed in Section \ref{S:Thestudyofconnectioningraphsthroughcovering-basedroughsets} to study the issue.
As is known, a graph without isolated vertices can induce a covering through Definition \ref{D:thefamilyinducedbygraphs}.
Combining with Proposition \ref{P:theconditionforthecircuitsofamatroidformacovering}, we know for a matroid $M$, when the graph $G(M)$ has no isolated vertices, there are two coverings of its vertex set, i.e. $\mathcal{C}(M)$ and $C(G(M))$.
Generally, these two coverings are different, but they can induce the same covering upper approximation operator.

\begin{lemma}
\label{L:therelationbetweenthecoveringapproximationinducedbyG(M)andthatofinducedbyM}
Let $M$ be a matroid.
If graph $G(M)$ has no isolated vertices, then for all $X \subseteq U(M)$, $\overline{C(G(M))}(X) = \overline{\mathcal{C}(M)}(X)$.
\end{lemma}

\begin{proof}
Since $G(M)$ has no isolated vertices, the families $C(G(M))$ and $\mathcal{C}(M)$ are two coverings of $U(M)$.
Thus $\overline{C(G(M))}$ and $\overline{\mathcal{C}(M)}$ are two covering upper approximation operators of $U(M)$, respectively.
Next, we prove these two operators are equal.
For all $x \in \overline{C(G(M))}(X)$, there exists $K \in C(G(M))$ such that $x \in K$ and $K \bigcap X \neq \emptyset$.
If $x \in X$, then $x \in \overline{\mathcal{C}(M)}(X)$ because $X \subseteq \overline{\mathcal{C}(M)}(X)$.
If $x \notin X$, pitch $y \in K \bigcap X$, then $x \neq y$ and $K = \{x, y\}$, i.e., $x$ is connected to $y$.
According to Proposition \ref{P:anequivalencecharacterizationforanytwovertexicesofthegraphinducedbymatroid}, there exists $C \in \mathcal{C}(M)$ such that $\{x, y\} \subseteq C$ which implies that $x \in \overline{\mathcal{C}(M)}(X)$.
Hence $\overline{C(G(M))}(X) \subseteq \overline{\mathcal{C}(M)}(X)$.
Conversely, for any $x \in \overline{\mathcal{C}(M)}(X)$, there exists $C \in \mathcal{C}(M)$ such that $x \in C$ and $C \bigcap X \neq \emptyset$.
If $x \in X$, then $x \in \overline{C(G(M))}(X)$ because $X \subseteq \overline{C(G(M))}(X)$.
If $x \notin X$, pitch $y \in C \bigcap X$, then $x \neq y$ and $\{x, y\} \subseteq C$, i.e., $x$ and $y$ are adjacent in graph $G(M)$.
Thus $\{x, y\} \in C(G(M))$ which implies $x \in \overline{C(G(M))}(X)$.
Hence $\overline{\mathcal{C}(M)}(X) \subseteq \overline{C(G(M))}(X)$.
\end{proof}

Therefore, the connectedness for the matroid, whose induced graph has no isolated vertices, can be characterized by the circuit covering.

\begin{theorem}
\label{T:theapproachestojustifytheconnectionofamatroid}
Let $M$ be a loopless matroid and the circuit family $\mathcal{C}(M)$ a covering of $U(M)$.
The following statements are equivalent:\\
(1) $M$ is connected.\\
(2) For any $\emptyset \neq X \subset U(M)$, $\overline{\mathcal{C}(M)}(X) \neq X$.\\
(3) For any $x \in U(M)$, $\overline{\mathcal{C}(M)}(x) = U(M)$.
\end{theorem}

\begin{proof}
$(1) \Leftrightarrow (2)$: $M$ is connected if and only if $G(M)$ is connected if and only if, for all $\emptyset \neq X \subset U(M)$, $\overline{C(G(M))}(X) \neq X$ if and only if, for all $\emptyset \neq X \subset U(M)$, $\overline{\mathcal{C}(M)}(X) \neq X$.

$(1) \Leftrightarrow (3)$: $M$ is connected if and only if $G(M)$ is a complete graph if and only if, for all $x \in U(M)$, $\overline{C(G(M))}(x) = U(M)$ if and only if, for all $x \in U(M)$, $\overline{\mathcal{C}(M)}(x) = U(M)$.
\end{proof}

From the above discussion, there are three steps to determine the connectedness of a matroid $M$.\\
\textbf{Step} 1: Judge whether $M$ has loops or not.\\
\textbf{Step} 2: Judge whether or not $\mathcal{C}(M)$ is a covering of $U(M)$.\\
\textbf{Step} 3: If the matroid has no loops and its circuit family forms a covering of its ground set, then utilize Theorem \ref{T:theapproachestojustifytheconnectionofamatroid} to determine the connectedness of it.

\section{Conclusions}
\label{S:conclusions}
We have discussed in this paper the issue of the connection of graphs which are undirected, simple and without isolated vertices in terms of covering-based rough sets.
Furthermore, the approaches to study the connection of graphs were applied to study that of matroids.
Based on the results of this paper, we intend designing efficient algorithms to determine the connection of a graph and that of a matroid, and will investigate some other problems of graph theory and matroid theory through rough sets.

\section{Acknowledgments}
This work is supported in part by the National Natural Science Foundation of China under Grant Nos. 61170128, 61379049, and 61379089, the Natural Science Foundation of Fujian Province, China, under Grant No. 2012J01294, the Science and Technology Key Project of Fujian Province, China, under Grant No. 2012H0043, and the Zhangzhou Research Fund under Grant No. Z2011001.



\end{document}